\documentclass[journal]{IEEEtran}

\ifCLASSINFOpdf
\else
   \usepackage[dvips]{graphicx}
\fi
\usepackage{url}
\usepackage[noadjust]{cite}
\input{mysymbol.sty}
\usepackage{mathtools}
\usepackage{multicol}
\usepackage{multirow}
\usepackage{subcaption}
\usepackage{booktabs}
\DeclarePairedDelimiterX{\norm}[1]{\lVert}{\rVert}{#1}
\usepackage{microtype}
\usepackage{graphicx}
\usepackage{xcolor}
\usepackage{algorithm2e}
\usepackage{booktabs} 
\usepackage{algpseudocode}
\usepackage{float}
\usepackage{balance}
\usepackage{graphicx}
\usepackage[utf8]{inputenc} 
\usepackage[T1]{fontenc}    
\usepackage{hyperref}       
\usepackage{url}            
\usepackage{booktabs}       
\usepackage{amsfonts}       
\usepackage{nicefrac}       
\usepackage{microtype}      
\usepackage{xcolor}         
\usepackage{multicol}
\usepackage{verbatim}

\usepackage{amssymb,amsmath,amsthm, amsfonts}

\newtheorem{proposition}{Proposition}
\newtheorem{remark}{Remark}

\definecolor{morange}{rgb}{0.8,0.2,0}
\definecolor{mblue}{rgb}{0,0.3,1.0}
\definecolor{mred}{rgb}{0.9,0.1,0.1}
\definecolor{mpurple}{rgb}{0 0 0}

\allowdisplaybreaks[4]
\begin{document}

\title{Fairness-aware Optimal Graph Filter Design}

\author{
\IEEEauthorblockN{O. Deniz Kose\IEEEauthorrefmark{1},} 
\and
\IEEEauthorblockN{Yanning Shen\IEEEauthorrefmark{1},}
\and 
\IEEEauthorblockN{and Gonzalo Mateos\IEEEauthorrefmark{2}}

\IEEEcompsocitemizethanks{\IEEEcompsocthanksitem\IEEEauthorrefmark{1}Department of Electrical Engineering and Computer Science, University of California Irvine, USA}

\IEEEcompsocitemizethanks{\IEEEcompsocthanksitem \IEEEauthorrefmark{2}Department of Electrical and Computer Engineering, University of Rochester, USA}

\IEEEcompsocitemizethanks{Preliminary ideas that inspired this work will appear at the 2023 Asilomar Conference on Signals, Systems, and Computers \cite{kose2023fairnessfilt}.}

\IEEEcompsocitemizethanks{Work in this paper was supported in part by the NSF awards CCF-1750428, CCF-1934962, and ECCS 2207457.}

}


\markboth{}
{Shell \MakeLowercase{\textit{et al.}}: Bare Demo of IEEEtran.cls for IEEE Journals}
\maketitle

\begin{abstract}
Graphs are mathematical tools that can be used to represent complex real-world interconnected systems, such as financial markets and social networks. Hence, machine learning (ML) over graphs has attracted significant attention recently. However, it has been demonstrated that ML over graphs amplifies the already existing bias towards certain under-represented groups in various decision-making problems due to the information aggregation over biased graph structures. Faced with this challenge, here we take a fresh look at the problem of bias mitigation in graph-based learning by borrowing insights from graph signal processing. Our idea is to introduce predesigned graph filters within an ML pipeline to reduce a novel unsupervised bias measure, namely the correlation between sensitive attributes and the underlying graph connectivity. We show that the optimal design of said filters can be cast as a convex problem in the graph spectral domain. We also formulate a linear programming (LP) problem informed by a theoretical bias analysis, which attains a closed-form solution and leads to a more efficient fairness-aware graph filter. Finally, for a design whose degrees of freedom are independent of the input graph size, we minimize the bias metric over the family of polynomial graph convolutional filters. Our optimal filter designs offer complementary strengths to explore favorable fairness-utility-complexity tradeoffs. For performance evaluation, we conduct extensive and reproducible node classification experiments over real-world networks. Our results show that the proposed framework leads to better fairness measures together with similar utility compared to state-of-the-art fairness-aware baselines.
\end{abstract}


\IEEEpeerreviewmaketitle

\section{Introduction}

We live in the era of connectivity, where the actions of humans and devices are increasingly driven by their relations to others. Concurrently, a significant amount of data describing different interconnected systems, such as social networks, the Internet of Things (IoT), the Web, and financial markets, is increasingly available. Processing and learning from such data can provide significant understanding and advancements for the corresponding networked systems \cite{kolaczyk2014statistical, hamilton2020graph}. In this context, machine learning (ML) over graphs has attracted increasing attention \cite{chami2022machine, gcn}, since graphs are widely utilized to represent complex underlying relations in real-world networks~\cite{mateos19spmag}. 

These relational patterns can be captured by graph edges, while attributes of nodes (nodal features) can be interpreted as signals defined on the vertices. For example, in a social network, user ages can be modeled as a graph signal, and the friendship information can be encoded by the edges. Graph signal processing (GSP) extends the tools in classical signal processing to graph signals \cite{gsp}, such as frequency analysis, sampling, and filtering \cite{gft, dgsp,  marques2017stationary, isufi2016autoregressive, isufi2022graph, romero2016kernel}. GSP and ML over graphs are closely intertwined, where the tools in one domain can be useful in the other one \cite{gsp, dong2020graph}. For instance, it has been demonstrated that graph neural networks (GNNs) can be designed, analyzed, and improved by leveraging GSP-based insights \cite{gama2020graphs, gama2020stability, dong2020graph}, which underscores the advancements that can be made  by cross-pollinating the findings in both domains. In this paper, we align with this vision and leverage GSP advances to enhance fairness in ML over graphs pipelines.\vspace{2pt}

\noindent\textbf{The pursuit of fairness in ML over graphs.}
Despite the growing interest in learning over graphs, the widespread deployment of these algorithms in real-world decision systems depends heavily on how socially responsible they are. Motivated by this concern, fairness in ML algorithms has attracted significant attention recently \cite{mehrabi2021survey, pessach2022review, holstein2019improving}. This work focuses on group fairness, which ensures that the learning algorithms incur no performance gap with respect to sensitive/protected attributes (such as ethnicity and religion). For example, the predictions of a job recommendation algorithm should be independent of the gender of applicants for a fair algorithm with respect to the sensitive attribute gender. Moreover,  throughout this paper, \emph{algorithmic bias} refers to the stereotypical correlations the learning algorithms encode and further propagate with respect to these sensitive attributes. Despite how critical the fairness of algorithms is for their applicability in real-world decision systems, several studies have demonstrated that ML models propagate the historical bias within the training data and lead to discriminatory results in ensuing applications \cite{ beutel2017data}. Specific to graph-based learning, the utilization of graph structure in the algorithm design has been shown to amplify the already existing bias \cite{fairgnn}. Recognizing these compounded challenges, recent works focus on fairness-aware learning over graphs and advocate different techniques to mitigate bias, such as adversarial regularization \cite{bose2019compositional, fairgnn}, fairness constraints \cite{buyl2021kl, kose2022fairnorm}, and fairness-aware graph data augmentation \cite{kose2022fair, spinelli2021fairdrop, dong2022edits}; see also Section II for additional discussion on related work.\vspace{2pt} 

\noindent\textbf{Proposed approach and innovations in context.} In this study, we advocate fairness-aware optimal graph filter designs. In order to mitigate bias derived from the graph topology, we subsequently introduce these predesigned filters within standard ML pipelines. To this end, we introduce a bias metric, $\rho$, which can be employed in graph-based unsupervised learning approaches and measures the linear correlation between an effective (filter-dependent) connectivity pattern and the sensitive attributes. 
We show that the $\rho$-minimizing optimal filter design can be cast as a convex problem in the graph spectral domain.
While this proposed approach is remarkably effective in mitigating graph-amplified biases, the total number of optimization variables is equal to the input graph size. Accordingly, solving the optimization problem becomes computationally expensive for large input graphs. For a more efficient fairness-aware solution, we carry out a bias analysis and upper bound the bias metric $\rho$ by bringing to bear GSP notions. Based on these theoretical findings, we formulate a novel linear programming (LP) filter design problem that attains a closed-form solution minimizing the derived upper bound. 
We finally propose a design whose degrees of freedom are independent of the size of the input graph, by minimizing $\rho$ over the family of polynomial graph convolutional filters. 

Our previous endeavor \cite{ouricassp} is also built upon spectral analysis of graph signals, where a fairness-aware \emph{dimensionality reduction} algorithm was developed. However, in \cite{ouricassp}, the information carried in certain frequencies is completely removed, which can adversely affect the overall utility (accuracy for node classification) of the underlying ML task. Instead, in the present work, we propose a suite of bias mitigation approaches to effectively filter out traces of the sensitive attribute signal (e.g., race, gender in social networks), while also offering the flexibility to delineate favorable fairness-utility-complexity tradeoffs in ML over graphs. Furthermore, unlike the intuitive but heuristic approach in the conference precursor to this paper \cite{kose2023fairnessfilt}, the fairness-aware graph filter designs proposed here are rooted on well-defined optimality criteria.\vspace{2pt}

\noindent\textbf{Summary of contributions.} Overall, our contributions are:\\
\textbf{i)} We introduce a novel, correlation-based bias metric for graphs, which can facilitate fairness-aware unsupervised learning from network data;\\
\textbf{ii)} We show that filtering  nodal representations which are obtained via graph aggregation can be used to manipulate the bias metric. An optimal graph filter is designed to minimize $\rho$ by solving a convex optimization problem in the spectral domain;\\
\textbf{iii)} For a more efficient bias mitigation solution, we upper bound $\rho$ by utilizing GSP-based tools and then minimize this surrogate cost, leading to an LP problem that attains a closed-form optimal solution. By restricting the search to the class of polynomial graph convolutional filters, the number of optimization variables decouples from the input graph size, and the resulting fairness-aware filters can be implemented in a distributed fashion; \\
\textbf{iv)} The novel filter designs are versatile and can be employed in different stages of the learning pipeline, as well as for various graph-based learning frameworks; and\\
\textbf{v)} Comprehensive experimental results for node classification on real-world networks corroborate the effectiveness of the proposed methods in mitigating bias while providing comparable utility to state-of-the-art fairness-aware baselines. In the interest of reproducible research, the code used to obtain all results in this paper is publicly available.

\noindent\emph{Notation:} The entries of a matrix $\mathbf{V}$ and a vector $\mathbf{v}$ are denoted by $V_{ij}$ and $v_i$, respectively. Calligraphic capital letters are utilized to represent sets. $\mathbf{I}_{N}$ refers to an $N\times N$ identity matrix. The notation $^\top$ stands for the transpose operation. For a vector $\bbv$, $\operatorname{diag}(\mathbf{v})$ represents a diagonal matrix whose $i$th diagonal entry equals to $v_i$. The $\ell_p-$norm of vector $\mathbf{v}$ is given by $\|\mathbf{v}\|_p:=\left(\sum_{i=1}^n\left|v_i\right|^p\right)^{1 / p}$.



\section{Related Work}
Here, we briefly review relevant related work to better position our contributions in context.

\subsection{Graph filters}
Extending classical signal processing tools to networked systems, graph filters are specific operators to manipulate graph signals. The existing literature generally focuses on linear graph filters represented by polynomials of a graph-shift operator \cite{dgsp, segarra2015distributed, segarra2015interpolation, gama2020graphs, isufi2022graph, romero2016kernel}. Graph filters are utilized for a number of applications, including but not limited to modeling the dynamics of opinion formation in social networks \cite{hegselmann2002opinion, patterson2010interaction}, or modeling the diffusion/percolation dynamics over networks \cite{segarra2015distributed, mei2015signal}. 
Recently, with the success of graph neural
networks (GNNs) for a number of graph-based tasks, graph filters have attracted increasing attention as the key component of GNNs \cite{gama2020graphs, ma2021unified, zhu2021interpreting, ruiz2021graph, isufi2021edgenets}. However, to the best of our knowledge, there has been no prior attempt to examine the benefits of pretrained filters towards decorrelating learned nodal representations from sensitive attributes. So far, optimal graph filter designs have not incorporated fairness criteria.


\subsection{Fairness-aware learning on graphs}
 In the fairness-aware graph-based learning domain, \cite{fairwalk} is a pioneering study that proposes a bias mitigation solution for random walk-based algorithms. Moreover, motivated by its success in general fairness-aware ML, adversarial regularization is also employed by several graph-based ML frameworks \cite{fairgnn, bose2019compositional, debiasing, guo2022learning}.   Specifically, \cite{fairgnn} focuses on partially available sensitive attributes, and \cite{debiasing} considers knowledge graphs. By modeling the sensitive attribute signal in the prior distribution, \cite{debayes} proposes a Bayesian strategy for fair node representation learning. In addition, \cite{subgroup} links the subgroup generalization to accuracy disparity based on a PAC-Bayesian analysis, while \cite{heterogeneous} presents multiple strategies to reduce the algorithmic bias in the representations of heterogeneous information networks. There is also a line of work that designs fair graph data augmentations to mitigate the bias within nodal features and the graph topology \cite{kose2022fair, kose2022fair2, nifty, spinelli2021fairdrop}. Finally, with a specific focus on link prediction, \cite{dyadic,all} introduce fairness-aware strategies that alter the adjacency matrix, while \cite{buyl2021kl} employs a fairness-aware regularizer. Unlike most of these works, the proposed strategies herein are based on a theoretical bias analysis and enjoy well-defined optimality. Furthermore, the collection of fairness-aware graph filters designed in Section \ref{sec:filter_design} can be employed in a versatile manner as both a pre-processing and post-processing operator in a number of graph-based learning environments; see also the numerical tests in Section \ref{sec:exps}.  
 While the draft of this paper was being finalized, we became aware of an interesting unpublished preprint~\cite{krasanakis2023graph} that explores fairness for GSP-based graph mining applications with a markedly different goal than ours. Indeed,~\cite{krasanakis2023graph} advocates a GNN framework as a surrogate of a fairness-aware graph filter, and here we design graph filters to mitigate bias in general ML on graphs pipelines. The approach in~\cite{krasanakis2023graph} is to ``edit'' the input graph signal for fairness enhancement and does not focus on optimal filter design. 
  Overall, our study is the first attempt to design fair graph filters to mitigate intrinsic bias by cross-pollinating the tools of GSP and ML over graphs. 

\begin{figure*}[!h]
	\begin{centering}
	\hspace{2cm} \includegraphics[width=0.8\textwidth]{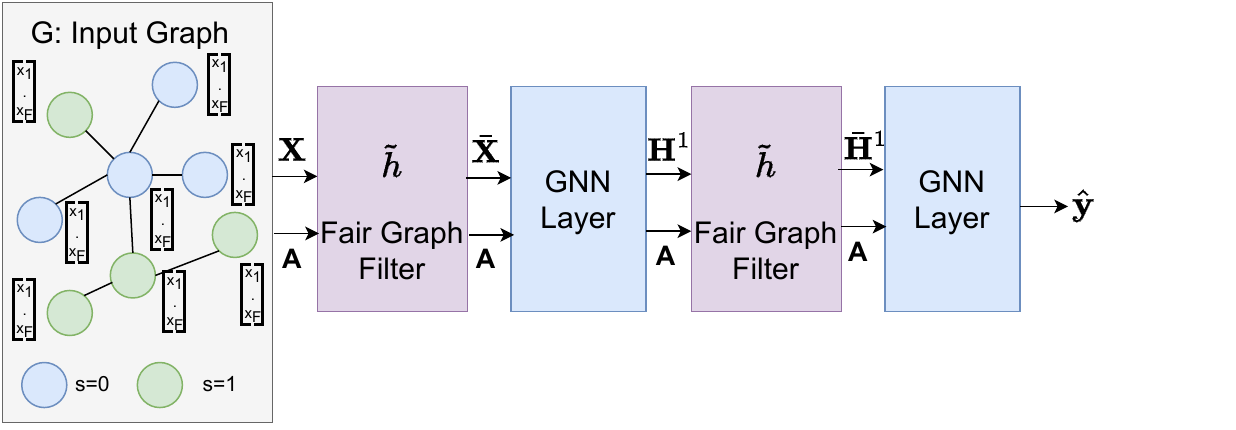} 
	\caption{The employment of a fairness-aware graph filter within a standard two-layer GNN-based learning pipeline as a pre-trained bias mitigation operator.}
	\label{fig:ex}
	\end{centering}
\end{figure*}

\section{Preliminaries and Problem Statement}
\label{sec:methodology}
\par The focus of this study is to mitigate bias in graph-based learning algorithms by employing graph filters for a given undirected graph $\mathcal{G}:=(\mathcal{V}, \mathcal{E})$, where $\mathcal{V}:=$ $\left\{v_{1}, v_{2}, \ldots, v_{N}\right\}$ denotes the set of nodes and $\mathcal{E} \subseteq \mathcal{V} \times \mathcal{V}$ is the set of edges. Connectivity of the input graph is encoded in the symmetric adjacency matrix $\mathbf{A} \in\{0,1\}^{N \times N}$, where $A_{i j}=1$ if and only if $\left(v_{i}, v_{j}\right) \in \mathcal{E}$. In addition, $\mathbf{X} \in \mathbb{R}^{N \times F}$ represents the nodal features of $\mathcal{G}$, whose columns are graph signals (one per feature). The diagonal degree matrix is $\mathbf{D} \in \mathbb{R}^{N \times N}$, where $D_{ii}$ denotes the degree of $v_{i}$. Let $\mathbf{L} = \mathbf{I}_{N}-\mathbf{D}^{-\frac{1}{2}}\mathbf{A}\mathbf{D}^{-\frac{1}{2}}$ denote the normalized graph Laplacian matrix, where the normalized adjacency matrix is represented by $\hat{\mathbf{A}}:=\mathbf{D}^{-\frac{1}{2}}\mathbf{A}\mathbf{D}^{-\frac{1}{2}}$. 

The \emph{sensitive attribute} is the nodal feature (such as ethnicity, religion) on which the decisions should not be dependent for fair decision-making. Herein, the sensitive attribute is assumed to be binary and is denoted by $\mathbf{s} \in \{-1,1\}^{N }$. The feature vector and the sensitive attribute of node $v_{i}$ are denoted by $\mathbf{x}_{i} \in \mathbb{R}^{F}$ and $s_{i} \in \{-1,1\}$, respectively. In (semi-supervised) node classification tasks, some vertices have (e.g., binary) labels $y_i$. For concrete examples of nodal features, labels, and sensitive attributes in several real-world network datasets, see Section \ref{subsec:datas}. 

\subsection{Graph signal processing fundamentals} 
The graph Fourier transform (GFT) is an orthonormal transform that provides the representation of a graph signal $\mathbf{z} \in \mathbb{R}^N$ in the graph spectral domain \cite{gft, gft2, gft3}. Specifically, taking the GFT of a graph signal amounts to projecting the signal onto a space spanned by the orthogonal eigenvectors of the positive semi-definite (PSD) normalized graph Laplacian matrix $\mathbf{L}$ \cite{gft}. Let the eigendecomposition of the normalized Laplacian be $\mathbf{L}=\mathbf{V} \mathbf{\Lambda} \mathbf{V}^{\top},$ where $\mathbf{\Lambda}=\textrm{diag}(\lambda_1,\ldots,\lambda_{N})$ collects the non-negative eigenvalues and $\mathbf{V}$ is the matrix of Laplacian eigenvectors. 
Then, the GFT of the graph signal $\mathbf{z} \in \mathbb{R}^{N}$ is given by $\Tilde{\mathbf{z}}=\mathbf{V}^{\top} \mathbf{z}$. Graph frequencies correspond to the eigenvalues of the Laplacian (a measure of smoothness of the eigenvectors with respect to the graph), meaning that the GFT decomposes signals into frequency modes (i.e., the eigenvectors of $\mathbf{L}$) of different variability over $\mathcal{G}$.

In classical signal processing, filters are utilized to manipulate signals such that their, e.g., unwanted components are attenuated or removed. Similarly, graph filters can be used to modify graph signals for different purposes, including graph signal classification \cite{zhu2003semi, belkin2004semi}, smoothing, and denoising \cite{zhang2008graph, shuman2011chebyshev}. Filtering an input graph signal $\mathbf{z}_{\text{in}} \in \mathbb{R}^N$ via a filter with frequency response $\tilde{\mathbf{h}}:=[\tilde{h}_{1}, \ldots,  \tilde{h}_{N}]^\top$ can be mathematically expressed as (e.g.,~\cite{gsp,gft,gama2020graphs})
\begin{equation}
\label{eq:id}
\begin{split}
    \mathbf{z}_{\text{out}}&= \mathbf{V} \underbrace{\operatorname{diag}(\tilde{h}_{1}, \ldots ,\tilde{h}_{N} ) \Tilde{\mathbf{z}}_{\text{in}}}_{\text {Frequency domain filtering}}.
\end{split}
\end{equation}
Therefore, filtering in the frequency domain corresponds to point-wise multiplication of the input signal's GFT, $\Tilde{\mathbf{z}}_{\text{in}}$, with the frequency response of graph filter, $\tilde{\mathbf{h}}$. Identity \eqref{eq:id} is akin to a convolution theorem for graph signals.


\subsection{Problem statement} 
In this paper, given $\mathcal{G}$ and $\mathbf{s}$, we address the problem of designing graph filters with frequency response $\tilde{\mathbf{h}} \in \mathbb{R}^{N}$, so that the bias caused by the graph topology can be attenuated with the application of the designed filters in the learning algorithm. A possible application of the fairness-aware graph filter in a GNN-based learning pipeline is depicted in Figure \ref{fig:ex}. As we elaborate in Section \ref{subsec:met}, bias attenuation will be pursued by minimization of a judicious bias metric; namely, the linear correlation between $\mathbf{s}$ and the effective graph aggregation operator that results upon filtering with $\tilde{\mathbf{h}}$.

\section{Bias Mitigation Rationale and Criterion}\label{sec:rationale_metric}
In this section, we first motivate our filtering approach for bias mitigation and provide a graph spectral domain illustration of the fairness-utility tradeoff. We then propose a filter-dependent bias measure that will serve as a criterion for our subsequent designs.
\subsection{Spectrum analysis}
\label{subsec:spec}
The homophily principle suggests that nodes with similar attributes are more likely to connect in networks, which hints at denser connectivity between the nodes with the same sensitive attributes and also with the same label \cite{segregation}. Hence, both the sensitive attributes $\mathbf{s}$ and node labels $\mathbf{y}$ are expected to be smooth signals over $\mathcal{G}$. In the GSP parlance, this implies higher energy concentration for $\Tilde{\mathbf{s}}$ and $\Tilde{\mathbf{y}}$ over lower frequencies. Now, we wish to design a filter that preserves the necessary information for a downstream task (node classification in this paper) after ``filtering out'' traces of the sensitive attribute.  
Naturally, the extent of the overlap between the spectra of  $\Tilde{\mathbf{s}}$ and $\Tilde{\mathbf{y}}$ plays an important role in the feasibility of said fairness-aware filter design. 

\begin{figure}[ht]
    \centering
        \includegraphics[width=\linewidth]{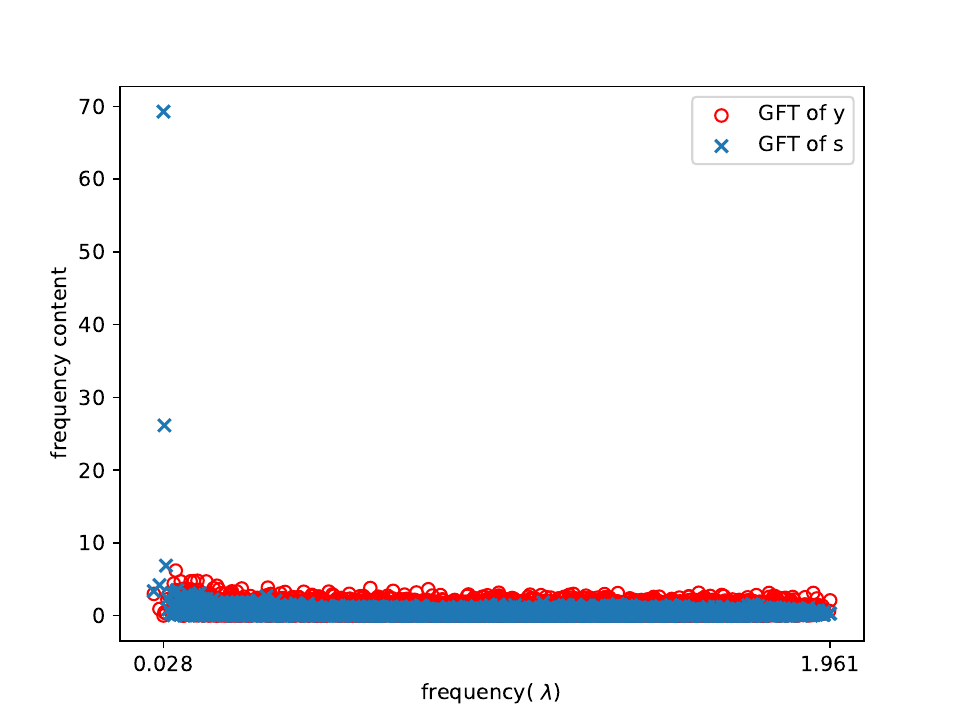}
        \includegraphics[width=\linewidth]{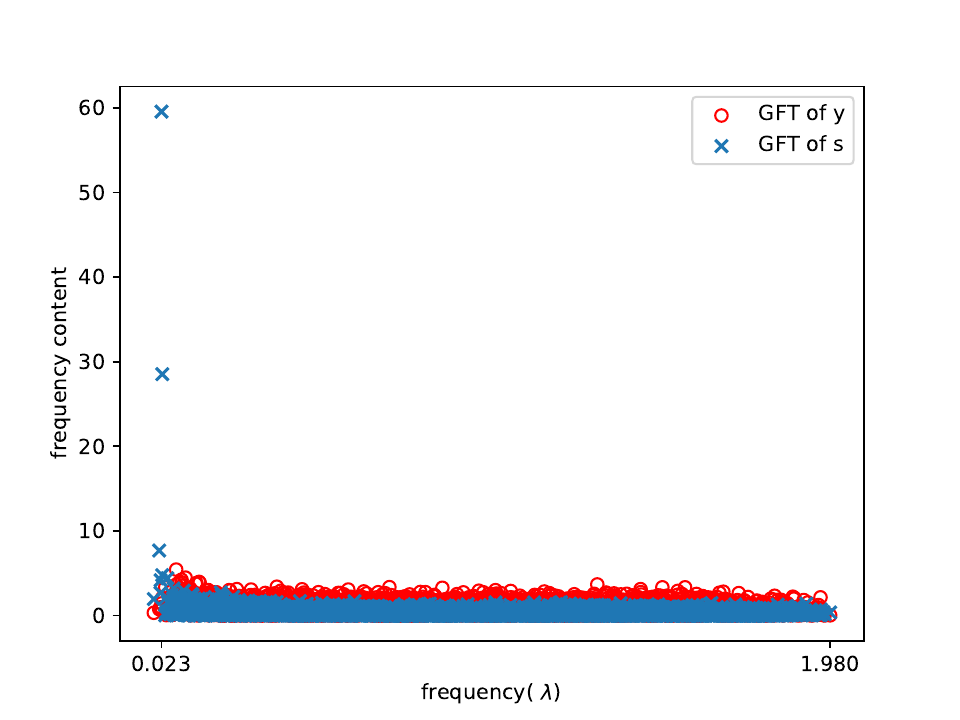}
    \caption{Spectra of the graph signals $\mathbf{s}$ (sensitive attributes) and $\mathbf{y}$ (labels) over different graph frequencies, for the real-world social network datasets (top) Pokec-z and (bottom) Pokec-n. Dataset statistics are presented in Table \ref{table:stats}. There are few low frequencies where the magnitudes of $\Tilde{\mathbf{s}}$ are markedly higher  than those of $\Tilde{\mathbf{y}}$.
    }
    \label{fig:within}
\end{figure}

To examine this tension, the GFT coefficients in $\Tilde{\mathbf{s}}$ and $\Tilde{\mathbf{y}}$ over different frequencies are depicted in Figure \ref{fig:within} for two real-world social networks with more than $6000$ nodes. For additional details of the datasets, see Section \ref{subsec:datas}. 
As expected, it can be observed that the spectra of $\Tilde{\mathbf{s}}$ and $\Tilde{\mathbf{y}}$ exhibit similar characteristics. However, there are certain (low) frequencies where the  magnitudes of $\Tilde{\mathbf{s}}$ are markedly higher than those of $\Tilde{\mathbf{y}}$. This subtle but important discordance (which is not just an artifact of these datasets) inspires our pursuit of frequency-selective graph filters for bias mitigation. The goal is to attenuate the sensitive information while preserving graph signals necessary for downstream ML tasks.

\subsection{Bias metric}
\label{subsec:met}
It has been demonstrated that leveraging graph structure in learning algorithms amplifies already existing bias due to the biased connectivity information \cite{fairgnn}. To exemplify this important point, in social networks, users (nodes) are often more likely to connect to other users with the same sensitive attributes (e.g., ethnicity, religion). This leads to denser connectivity between the nodes from the same sensitive groups, and hence a graph structure that is highly correlated with the sensitive attributes \cite{segregation}. 
Motivated by this, the linear correlation between the sensitive attribute signal $\bbs$ and graph topology is considered for the ensuing bias analysis and mitigation strategy. 

Several graph-based learning approaches rely on node representations obtained via local aggregation of information (possibly followed by a pointwise non-linearity) \cite{gcn, gsp}. In the simplest possible terms, this process can be summarized as
\begin{equation}
\label{eq:aggr}
\mathbf{R}=\hat{\mathbf{A}}\mathbf{X},
\end{equation}
where $\mathbf{R}$ denotes the aggregated node representations, and $\mathbf{X}$ is the input graph signal (or the representations from the previous layer as in Figure \ref{fig:ex}); see, e.g.,~\cite{chami2022machine, gcn}. In \eqref{eq:aggr}, we have purposely omitted learnable weights to simplify the notation while retaining the components essential to our argument. Hence, if a filtered graph signal $\mathbf{\bar{X}}= \mathbf{V} \text{diag}(\tilde{\mathbf{h}}) \mathbf{V}^{\top} \mathbf{X}$ is input, the obtained representation becomes
\begin{equation}
\label{eq:effec}
    \begin{split}
        \mathbf{R}^{\text{f}}&=\hat{\mathbf{A}}\mathbf{\bar{X}}\\
        &=\mathbf{V} (\mathbf{I}_{N} - \boldsymbol{\Lambda})\mathbf{V}^{\top}\mathbf{\bar{X}}\\
       &=\mathbf{V} (\mathbf{I}_{N} - \boldsymbol{\Lambda} )\mathbf{V}^{\top} \mathbf{V} \text{diag}(\tilde{\mathbf{h}}) \mathbf{V}^{\top} \mathbf{X}\\
        &=\mathbf{V} (\mathbf{I}_{N} - \boldsymbol{\Lambda} )\text{diag}(\tilde{\mathbf{h}}) \mathbf{V}^{\top} \mathbf{X}\\
        &=\mathbf{\bar{A}} \mathbf{X},
    \end{split}
\end{equation}
where $\mathbf{\bar{A}}:=\mathbf{V} (\mathbf{I}_{N} - \boldsymbol{\Lambda} )\text{diag}(\tilde{\mathbf{h}}) \mathbf{V}^{\top}$. 

Therefore, if a filtered signal $\mathbf{\bar{X}}$ is fed to the aggregation process, the effective network operator that is utilized in the information aggregation becomes $\mathbf{\bar{A}}$. Building on this quite simple but key observation, the linear correlation between the sensitive attributes $\mathbf{s}$ and $\mathbf{\bar{A}}$ is hence employed as a bias measure. This correlation is proportional to $\mathbf{s}^{\top} \mathbf{\bar{A}}_{:,i}$, for the $i$th column of $\mathbf{\bar{A}}$. Overall, we aim at minimizing the \emph{total correlation} \cite{kose2022fair} between $\mathbf{\bar{A}}$ and $\mathbf{s}$, which we denote as $\rho :=\|\mathbf{s}^{\top} \mathbf{\bar{A}}\|_{2}$. Notice that $\rho=\rho(\tilde{\mathbf{h}})$ because $\mathbf{\bar{A}}=\mathbf{V} (\mathbf{I}_{N} - \boldsymbol{\Lambda} )\text{diag}(\tilde{\mathbf{h}}) \mathbf{V}^{\top}$, hence we can search over filter frequency responses to reduce graph-induced bias. This filter design problem is the subject deal with next.


\section{Fair Graph Filter Designs}\label{sec:filter_design}

\subsection{Direct optimization of $\rho$ } 
\label{subsec:direct}
Here we describe our convex optimization framework for fairness-aware optimal graph filter design.
The idea is to formulate the following optimization problem to reduce the bias metric $\rho =\|\mathbf{s}^{\top} \mathbf{\bar{A}}\|_{2}$ via the employment of a graph filter with frequency response $\tilde{\mathbf{h}}$:
\begin{equation}
\label{eq:direct}
    \begin{aligned}
\tilde{\mathbf{h}}^{\text{f}}:=  & \:\underset{{\mathbf{\tilde{h}}}}{\operatorname{argmin}} \hspace{0.2cm} \rho(\tilde{\mathbf{h}}) \\
 \text { s. to }  & \rho(\tilde{\mathbf{h}})=\|\mathbf{s}^{\top} \mathbf{V} (\mathbf{I}_{N} - \boldsymbol{\Lambda} )\text{diag}(\tilde{\mathbf{h}}) \mathbf{V}^{\top} \|_{2},\\
 &\sum_{i=1}^{N} \tilde{h}_{i} \geq N \tau,\\
& 
0 \leq \tilde{h}_{i} \leq 1, \forall i \in \{1, \dots, N\}.
\end{aligned}
\end{equation}
While we have discussed the criterion at length, the constraints deserve justification. Here, $\tau$ is a hyperparameter to control the amount of filtered information. It is important to emphasize that $\rho$ can be minimized by setting $\tilde{\mathbf{h}}^{\text{f}}=\mathbf{0}$, which is equivalent to filtering out all information. This trivial solution is fair but naturally undesirable, because it sacrifices all utility. As we argued in Section \ref{subsec:spec}, there needs to be a trade-off between utility and fairness. This trade-off can be empirically adjusted via the design parameter $\tau$. 
Furthermore, the entries of $\tilde{\mathbf{h}}^{\text{f}}$ are constrained to not exceed $1$. The spectrum of the input graph signal does not change for those frequencies $\lambda_{i}$, where $\tilde{h}_{i}=1$. Thus, this constraint is utilized to preserve information in the frequencies that do not propagate bias as dictated by $\rho$. Overall, this choice is motivated by utility considerations in the downstream tasks. Note that the formulation in \eqref{eq:direct} is convex for the specified constraints; thus, it can be solved to global optimality using off-the-shelf methods. 

\begin{remark}[Spectral-domain design and eigendecomposition]\label{rem:eigenvectors}\normalfont 
The advocated graph spectral-domain design of the bias mitigating filter necessitates computing an eigendecomposition of the normalized Laplacian $\mathbf{L}$ prior to optimization. This $O(N^3)$ step can certainly challenge the applicability of the proposed approach when it comes to learning over large-scale graphs. This limitation nonwithstanding, our experimental results demonstrate this framework can comfortably handle network datasets with several thousands of nodes. Follow-up work on eigendecomposition-free filter designs in the vertex domains is certainly of interest; see also the related discussion preceding Remark \ref{rem:flexible}.
\end{remark}

\subsection{Linear programming with closed-form solution}\label{subsec:LP}
The formulation in \eqref{eq:direct} involves the optimization of $N$ variables, which incurs high complexity for large graphs. To sidestep this potential computational bottleneck, we derive a surrogate cost that is amenable to efficient minimization. Specifically, we first conduct a bias analysis and upper bound the bias metric $\rho$. We then show that minimization of the upper bound results in an LP, whose solution is a filter with frequency response $\tilde{\mathbf{h}}_{cf}^{\text{f}}$. Remarkably, the solution to the LP is given in closed form, and once more, it can effectively ``filter out" the sensitive information from the bias-amplifying graph connectivity.

First, Proposition \ref{prop:corr1} reveals the sources of bias and provides an upper bound on the total correlation between $\mathbf{s}$ and $\mathbf{\bar{A}}$.
\begin{proposition}
\label{prop:corr1}
Consider filtering signals using a graph filter with frequency response $\tilde{\mathbf{h}}$ prior to aggregation using $\hat{\mathbf{A}}$, and let $\mathbf{\bar{A}}:=\mathbf{V} (\mathbf{I}_{N} - \boldsymbol{\Lambda} )\text{diag}(\tilde{\mathbf{h}}) \mathbf{V}^{\top}$. Then, $\rho :=\|\mathbf{s}^{\top} \mathbf{\bar{A}}\|_{2}$ can be upper bounded by
\begin{equation}\label{rho}
  \rho \leq \sqrt{N} \sum_{i=1}^{N} |\tilde{s}_i| |(1-\lambda_{i})| |\tilde{h}_{i}|.
\end{equation}
\end{proposition}
\begin{proof}
Leveraging the definitions of $\rho$ and the effective aggregation operator $\bar{\mathbf{A}}$, we have $\mathbf{\bar{A}}:= \mathbf{V} (\mathbf{I}_{N} - \boldsymbol{\Lambda} )\text{diag}(\tilde{\mathbf{h}}) \mathbf{V}^{\top}$:
\begin{equation}
\label{eq:f}
    \begin{split}
       \rho &=\|\mathbf{s}^{\top} \mathbf{\bar{A}}\|_{2}\\
        &=\|\mathbf{s}^{\top} \mathbf{V} (\mathbf{I}_{N} - \boldsymbol{\Lambda} )\text{diag}(\tilde{\mathbf{h}}) \mathbf{V}^{\top}\|_{2}.
    \end{split}
\end{equation}
Furthermore, \eqref{eq:f} can be reformulated based on the definition of GFT for the sensitive attribute signal $\mathbf{s}$:
\begin{equation}
         \rho=\|\tilde{\mathbf{s}}^{\top}(\mathbf{I}_{N} - \boldsymbol{\Lambda} )\text{diag}(\tilde{\mathbf{h}}) \mathbf{V}^{\top}\|_{2}.
\end{equation}
By utilizing the norm inequality, $\rho$ can be upper bounded:
\begin{equation}
    \begin{split}
      \rho& \leq \|\tilde{\mathbf{s}}^{\top}(\mathbf{I}_{N} - \boldsymbol{\Lambda} )\text{diag}(\tilde{\mathbf{h}}) \mathbf{V}^{\top}\|_{1}\\
         &= \sum_{j=1}^{N}| \sum_{i=1}^{N}\tilde{s}_i (1-\lambda_{i}) \tilde{h}_{i} v_{ji}|.  
    \end{split}
\end{equation}
Based on the triangle inequality, the following inequality can further be derived:
\begin{equation}
    \begin{split}
          \rho& \leq \sum_{j=1}^{N} \sum_{i=1}^{N}|\tilde{s}_i (1-\lambda_{i}) \tilde{h_{i}} v_{ji}|\\
         &=  \sum_{j=1}^{N} \sum_{i=1}^{N} |\tilde{s}_i| |(1-\lambda_{i})| |\tilde{h_{i}}| |v_{ji}|\\
         &= \sum_{i=1}^{N} |\tilde{s}_i| |(1-\lambda_{i})| |\tilde{h}_{i}| \sum_{j=1}^{N} |v_{ji}|\\
         &= \sum_{i=1}^{N} |\tilde{s}_i| |(1-\lambda_{i})| \tilde{h}_{i} \|\mathbf{V}_{:,i}\|_{1}
    \end{split}
\end{equation}  
Moreover, the relation between the $\ell_{1}$ and $\ell_{2}$-norms of a vector $\mathbf{a} \in \mathbb{R}^{N}$ can be written as $\|a\|_{1} \leq \sqrt{N} \|a\|_{2}$, based on which it follows that:
\begin{equation}
\label{eq:mfin}
    \begin{split}
          \rho &\leq \sum_{i=1}^{N} |\tilde{s}_i| |(1-\lambda_{i})|\tilde{ h}_{i} \|\mathbf{V}_{:,i}\|_{1}\\
          &\leq \sqrt{N} \sum_{i=1}^{N} |\tilde{s}_i| |(1-\lambda_{i})| \tilde{h}_{i} \|\mathbf{V}_{:,i}\|_{2}\\
           &=\sqrt{N} \sum_{i=1}^{N} |\tilde{s}_i| |(1-\lambda_{i})| \tilde{h}_{i}, 
    \end{split}
\end{equation}
where the last equality holds because the eigenvectors of $\mathbf{L}$ are orthonormal.
\end{proof}
%
Proposition \ref{prop:corr1} shows that the linear correlation between the effective graph topology and the sensitive attributes is a function of $\sum_{i=1}^{N} |\tilde{s}_i| |(1-\lambda_{i})| |\tilde{h}_{i}|$. Therefore, we can design a ``matched" graph filter to reduce this term and hence the bias. Define $m_{i} := |\tilde{s}_i| |(1-\lambda_{i})| $, for all $i=1,\ldots,N$, and let $\mathbf{m}\in \mathbb{R}^N$ be the vector whose $i$th component is $m_i$. Then, the following LP problem can be formulated for the design of an optimal fair graph filter:
%
%
\begin{equation}
\label{eq:lp}
    \begin{aligned}
\tilde{\mathbf{h}}_{cf}^f :=  & \:\underset{{\tilde{\mathbf{h}}}}{\operatorname{argmin}} \hspace{0.2cm} \mathbf{m}^\top \tilde{\mathbf{h}} \\
 \text { s. to }  
 &\sum_{i=1}^{N} \tilde{h}_{i} \geq N \tau,\\
& 
0 \leq \tilde{h}_{i} \leq 1, \forall i \in \{1, \dots, N\}.
\end{aligned}
\end{equation}
The same set of constraints as in \eqref{eq:direct} are employed here. Let $\boldsymbol{\alpha} = \operatorname{argsort}(-\mathbf{m})$ be the vector containing the indices of the elements in $\mathbf{m}$ sorted in descending order. The closed-form solution for this LP problem can be obtained as:
\begin{equation}
\label{eq:filter}
    (\tilde{h}^{\text{f}}_{cf})_{\alpha_{i}}= \left[1 - \left[N(1-\tau) - \sum_{j=1}^{i-1} \left(1 - (\tilde{h}^{\text{f}}_{cf})_{\alpha_{j}}\right)\right]_+\right]_+,
\end{equation}
where $[x]_+:=\max(0,x)$ is a projection operator onto the non-negative reals.

\begin{proof}[Proof (sketch)]
It always holds that $m_i \geq 0$ and $\tilde{h}_{i} \geq 0$, for all $i=1, \dots, N$, due to the definition of $\mathbf{m}$ and the box constraints on each of the $\tilde{h}_{i}$. Therefore, the cost function is always non-negative, i.e., $\mathbf{m}^\top \tilde{\mathbf{h}} \geq 0$, and the equality is achieved when $\tilde{\mathbf{h}}= \mathbf{0}$. However, such a solution does not satisfy the constraint that lower bounds the sum of elements in $\tilde{\mathbf{h}}$ (unless when $\tau=0$, but as discussed in Section V-A, this case is of no practical interest). The conclusion is that the optimal solution is attained on the boundary of the feasible set, where $\tilde{h}_{i}$ takes the smallest possible values for the largest entries of $\mathbf{m}$, as long as the first constraint holds. Specifically, the optimal $\tilde{\mathbf{h}}$ has null entries (or entries that are smaller than $1$) in the indices where vector $\mathbf{m}$ takes the largest values as long as the filtering budget (imposed by the first constraint) is not exhausted, which provides the recursive solution in \eqref{eq:filter}.
\end{proof}

For the budget prescribed by $\tau$, the recursive definition of the filter's frequency response in \eqref{eq:filter} specifies the optimal solution of the LP design in \eqref{eq:lp}.

\subsection{Polynomial graph convolutional filter}
\label{subsec:conv}
The LP-based filter design in the previous section admits a closed-form solution, and accordingly, it offers computational savings relative to \eqref{eq:direct}, since solving the latter necessitates an iterative procedure. 
Here, instead, we adopt a polynomial graph filter parameterization \cite{isufi2022graph}, which offers an explicit handle on the number of optimization variables. This way, the number of variables decouples from (and can be markedly smaller than) the size of $\mathcal{G}$. 

Polynomial graph filters are often the operators of choice in several SP and ML tasks due to their parameter sharing property, locality and linear computational complexity. When these filters are used in GNNs, the parameter sharing property allows them to learn complex relations within graphs (including large-scale ones) based on a limited number of training samples. The locality property implies they can be implemented in a distributed fashion, solely via exchanges of information with neighboring nodes in the graph. Finally, their linear computational complexity aids scalability \cite{isufi2022graph}. 
Polynomial graph convolutional filters are linear mappings, $\mathbf{z}_{out}=\mathbf{H}\mathbf{z}_{in}$, between graph signals, where
\begin{equation}
\label{eq:fir}
    \mathbf{H}:=\sum_{l=0}^{L-1} h_l \mathbf{\hat{A}}^l.
\end{equation}
Here, $\mathbf{h}:=\left[h_0, \ldots, h_{L-1}\right]^\top$ are the filter coefficients with $L-1$ denoting the filter order, and $\mathbf{\hat{A}}$ is selected as the graph-shift operator \cite{dgsp} in our design. Notice how \ref{eq:fir} resembles a finite impulse response (FIR) filter, with the identification of $\hat{\mathbf{A}}^l$ as an $l$th-order shift operator acting on graph signals \cite{isufi2022graph}.

In the frequency domain, the filter's response is given by $\tilde{\mathbf{h}}:=\boldsymbol{\Psi} \mathbf{h}$, for the $N \times L$ Vandermonde matrix $\Psi$, where $\Psi_{i j}:=\left(1 - \Lambda_{i i}\right)^{j-1}$ \cite{sandryhaila2014discrete}. Based on this parameterization, the optimization problem in \eqref{eq:direct} can be reformulated as:
\begin{equation}
\label{eq:conv}
    \begin{aligned}
\mathbf{h}^{\text{f}}:=  &\: \underset{{\mathbf{h}}}{\operatorname{argmin}} \hspace{0.2cm} \rho(\mathbf{h}) \\
 \text { s. to }  & \rho(\mathbf{h}) = \|\mathbf{s}^{\top} \mathbf{V} (\mathbf{I}_{N} - \boldsymbol{\Lambda} )\text{diag}(\boldsymbol{\Psi} \mathbf{h}) \mathbf{V}^{\top} \|_{2}, \\
 &\sum_{i=1}^{N} (\boldsymbol{\Psi} \mathbf{h})_{i} \geq N \tau,\\
& 
0 \leq (\boldsymbol{\Psi} \mathbf{h})_{i} \leq 1, \forall i \in \{1, \dots, N\}.
\end{aligned}
\end{equation}
The number of optimization variables is $L$, regardless of the number of nodes in $\mathcal{G}$. By selecting a filter order that satisfies $L \ll N$, this approach can be a better fit for large graphs. Meanwhile, the fairness improvement it provides may be limited when compared to our previous designs, as its degrees of freedom are purposedly reduced. Another salient feature of the polynomial graph filter \eqref{eq:fir} obtained by solving \eqref{eq:conv} is that it can be \emph{directly implemented} in the vertex domain via (distributed) information exchanges among neighbors. Polynomial parameterizations of the filters $\tilde{\mathbf{h}}^f$ and $\tilde{\mathbf{h}}_{cf}^f$ can be obtained as well, because they are jointly diagonalizable with $\hat{\mathbf{A}}$ by construction~\cite[Prop. 1]{segarra2015distributed}. However, this requires extra computation to interpolate the designed frequency responses and will likely necessitate a high value of $L$.

Note that $\rho$ can also be optimized in the vertex domain by minimizing $\rho=\|\mathbf{s}^{\top} \mathbf{\bar{A}}\|_{2}= \|\mathbf{s}^{\top} \mathbf{\hat{A}} \mathbf{H}\|_{2} = \|\sum_{l=0}^{L-1} h_l \mathbf{s}^{\top} \mathbf{\hat{A}}^{l+1}\|_{2}$ with respect to the graph filter coefficients $\mathbf{h}$. This way, one eliminates the need to calculate eigenvectors and eigenvalues of the graph Laplacian (cf. Remark \ref{rem:eigenvectors}). However, the design of constraints for this formulation is less intuitive than in the frequency domain and becomes non-trivial. We leave this interesting endeavor as a future research direction.\\

\begin{remark}[Flexible use of the proposed filter designs]\label{rem:flexible}\normalfont
The designed fair filters $\tilde{\mathbf{h}}^{\text{f}}$, $\mathbf{h}^{\text{f}}$, and $\tilde{\mathbf{h}}_{cf}^{\text{f}}$ can be employed in a flexible way to mitigate bias for different graph-based learning algorithms. They can be applied to the graph signals that are input to or output from the learning algorithms. Models designed for attributed graphs generally utilize the information from both the nodal features and graph topology \cite{gcn}. Thus, the proposed filters can be applied to the nodal features before they are fed to the learning pipeline in order to prevent the amplification of bias due to the graph connectivity. Alternatively, for any algorithm that outputs a graph signal (e.g., node labels in node classification), $\tilde{\mathbf{h}}^{\text{f}}$, $\mathbf{h}^{\text{f}}$, and $\tilde{\mathbf{h}}_{cf}^{\text{f}}$ can be employed on the output graph signal as fairness-aware post-processing operators. Overall, the impact of the proposed fair filter designs can permeate several GNN-based learning frameworks in a versatile manner. 
\end{remark}


\subsection{Discussion}\label{subsec:disc}
We have proposed three novel designs with complementary strengths to mitigate bias in the network topology via graph filtering. Each design has certain advantages over the others when it comes to manipulating the effect that input graph structures and sensitive attributes have on learned representations. In the first design, a fair graph filter $\tilde{\mathbf{h}}^{\text{f}}$, is obtained by solving a convex optimization problem that directly minimizes the bias measure $\rho$. Compared to $\tilde{\mathbf{h}}^{\text{f}}_{cf}$ whose design is based on an upper bound on $\rho$, $\tilde{\mathbf{h}}^{\text{f}}$ is expected to yield better bias mitigation performance, especially when the bound gets looser for the input graph. Moreover, as $\tilde{\mathbf{h}}^{\text{f}}$ has higher degrees of freedom than the polynomial filter $\mathbf{h}^{\text{f}}$, again its application is expected to decrease $\rho$ in a more effective way. On the other hand, both $\tilde{\mathbf{h}}^{\text{f}}_{cf}$ and $\mathbf{h}^{\text{f}}$ provide computationally more efficient bias mitigation solutions than $\tilde{\mathbf{h}}^{\text{f}}$. Furthermore, while $\tilde{\mathbf{h}}^{\text{f}}_{cf}$ is given in closed form and thus eliminates the need for iterative solvers, the number of optimization variables in the problem defining $\mathbf{h}^{\text{f}}$ is independent of the input graph size (the complexity of the sorting operation in the recursive computation of $\tilde{\mathbf{h}}^{\text{f}}_{cf}$ still grows with $N$). Granted, the number of constraints in \eqref{eq:conv} does depend on $N$, and that is why a full-blown vertex domain formulation is still of interest; see the discussion preceding Remark \ref{rem:flexible}. All in all, both $\tilde{\mathbf{h}}^{\text{f}}_{cf}$ and $\mathbf{h}^{\text{f}}$ can provide the most efficient solution based on the input graph properties. 

Overall, all our proposed fairness-aware graph filter designs can be employed in a flexible and efficient manner in several graph-based ML frameworks. For example, within GNN structures, these filters can be utilized as \emph{pre-trained} bias mitigation operators before each GNN layer, e.g., see Figure \ref{fig:ex}. It is important to emphasize that the employment of these filters as bias mitigation sub-layers within NNs does not modify the training process, unlike the majority of existing approaches that utilize fairness-aware regularizers and constraints \cite{fairgnn, bose2019compositional, debiasing, guo2022learning, dyadic, buyl2021kl}. Therefore, our filters can lead to more stable training compared to these strategies, especially adversarial regularization-based ones that are known to suffer from instability issues \cite{kodali2017convergence}. Moreover, the proposed filters need to be computed only once for a given $\mathcal{G}$, after which they can be utilized for various tasks on said graph.

\section{Experimental Results}
\label{sec:exps}
\subsection{Dataset and experimental setup}
\label{subsec:datas}
\begin{table}[h]
	\centering
\caption{Dataset statistics. }
\label{table:stats}
\begin{tabular}{c c c c c c c}
\toprule
                                                    Dataset &  $|\mathcal{S}_{-1}|$ & $|\mathcal{S}_{1}|$ & $|\mathcal{Y}_{-1}|$ & $|\mathcal{Y}_{1}|$ 
                                                    & $|\mathcal{E}|$                   \\ 
\midrule
Pokec-z  & $4851$ & $2808$ & $3856$ & $3803$ & $29476$ \\
Pokec-n  & $4040$ & $2145$ & $3432$ & $2753$ & $21844$ \\
\bottomrule
\end{tabular}
\end{table}

\begin{table*}[ht]
	\centering
\caption{Proposed filters as bias mitigation layers in a GNN model.}

\label{table:gnn}
\begin{scriptsize}
\begin{tabular}{l c c c c c c}
\toprule
                                                    & \multicolumn{3}{{c}}{Pokec-z}  & \multicolumn{3}{{c}}{Pokec-n}                                 \\ 
\cmidrule(r){2-4} \cmidrule(r){5-7}
                       & Accuracy ($\%$) & $\Delta_{S P}$ ($\%$) & $\Delta_{E O}$ ($\%$) & Accuracy ($\%$) & $\Delta_{S P}$ ($\%$) & $\Delta_{E O}$ ($\%$) \\\midrule
{GNN} 
                   & $ \textbf{66.52} \pm 0.27$ & $6.79 \pm 2.45$  & $7.26 \pm 3.29$ & $ 64.96 \pm 0.19$ & $6.79 \pm 2.45$  & $7.26 \pm 3.29$ 

\\ \cmidrule(r){1-7}  
{Adversarial} 
                   & $  64.26 \pm 1.79$ & $4.85 \pm 2.16$  & $5.99 \pm 2.71$ & $  64.22 \pm 0.71$ & $4.34 \pm 3.87$  & $3.84 \pm 2.71$   
    \\ \cmidrule(r){1-7}  
{\color{mpurple}EDITS \cite{dong2022edits}} 
                   & $ {\color{mpurple} 62.67 \pm 2.64}$ & ${\color{mpurple}3.17 \pm 2.49}$  & ${\color{mpurple}4.54 \pm 2.99}$ & $ {\color{mpurple} 62.67 \pm 0.51}$ & ${\color{mpurple}4.40 \pm 2.41}$  & ${\color{mpurple}5.38 \pm 1.92}$   
    \\ \cmidrule(r){1-7}  
{\color{mpurple}FairDrop \cite{spinelli2021fairdrop}} 
                   & $ {\color{mpurple} 66.79 \pm 0.65}$ & ${\color{mpurple}9.11 \pm 1.89}$  & ${\color{mpurple}8.35 \pm 3.81}$ & $ {\color{mpurple} 64.33 \pm 0.44}$ & ${4.46 \pm 1.67}$  & $5.02 \pm 1.84$   
   \\ \cmidrule(r){1-7}  
 {  $\tilde{\mathbf{h}}^{\text{f}}+ $ GNN } 
                  & $  66.05 \pm 0.30$ & $\mathbf{1.08} \pm 1.20$  & $2.20 \pm 2.06$  & $  \mathbf{65.07} \pm 0.21$ & $\mathbf{2.12} \pm 1.01$  & $\mathbf{2.42} \pm 1.96$\\ \cmidrule(r){1-7} 
                  
{$\tilde{\mathbf{h}}_{cf}^{\text{f}} + $ GNN } 
                   & $  66.34 \pm 0.27$ & $1.23 \pm 1.43$  & $\mathbf{2.15} \pm 1.96$ & $  \mathbf{65.05} \pm 0.21$ & $\mathbf{2.13} \pm 0.93$  & $\mathbf{2.39} \pm 1.78$ \\ \cmidrule(r){1-7}
{  $\mathbf{h}^{\text{f}}+ $ GNN} 
                  & $  66.32 \pm 0.27$ & $3.36 \pm 1.99$  & $4.21 \pm 2.43$  & $  \mathbf{65.07} \pm 0.21$ & $4.39 \pm 2.01$  & $5.13 \pm 2.00$\\


\bottomrule
\end{tabular}
\end{scriptsize}
\end{table*}

\begin{table}[ht]
	\centering
\caption{Total Pearson correlation coefficients \cite{kose2022fair} between representations and sensitive attributes before/after $\mathbf{\tilde{h}}^{f}$.}

\label{table:corr}
\begin{scriptsize}
\begin{tabular}{l c c c c}
\toprule
                                                    & \multicolumn{2}{{c}}{Pokec-z}  & \multicolumn{2}{{c}}{Pokec-n}                                 \\ 
\cmidrule(r){2-3} \cmidrule(r){4-5}
                       & Before $\mathbf{\tilde{h}}^{f}$ & After $\mathbf{\tilde{h}}^{f}$ & Before $\mathbf{\tilde{h}}^{f}$ & After $\mathbf{\tilde{h}}^{f}$ \\\midrule
 {  $1$st layer }   & $  4.27$ & $1.98$  & $  4.60$ & $1.94$  
                  \\ \cmidrule(r){1-5} 
                   {  $2$nd layer} 
                  & $ 4.25 \pm 0.03$ & $2.96 \pm 0.01$   & $ 3.12 \pm 0.05$ & $2.21 \pm 0.02$ 
                  \\
\bottomrule
\end{tabular}
\end{scriptsize}
\end{table}

\noindent\textbf{Datasets.} The performance of the proposed fair filter designs is evaluated on the node classification task over real-world social networks Pokec-z and Pokec-n. Pokec-z and Pokec-n are the sampled versions of the 2012 Pokec network \cite{pokec}, which is a Facebook-like social network in Slovakia \cite{fairgnn}. The region of the users is utilized as the sensitive attribute, where the users are from two major regions. Labels for the node classification task are the binarized working field of the users.
Statistics for the utilized datasets are presented in Table \ref{table:stats}, where $\mathcal{S}_{i}$ and $\mathcal{Y}_{i}$ represent the set of nodes with sensitive attribute and class label $i$, respectively. Note that $N=|\mathcal{S}_{-1}|+|\mathcal{S}_{1}| =|\mathcal{Y}_{-1}|+|\mathcal{Y}_{1}|$.\vspace{2pt}

\noindent \textbf{Evaluation metrics.} Accuracy is adopted as the utility metric of node classification. For fairness assessment, two quantitative measures of group fairness metrics are reported, namely \textit{statistical parity}: $\Delta_{S P}=|P(\hat{y}=1 \mid s=-1)-P(\hat{y}=1 \mid s=1)|$ and \textit{equal opportunity}: $\Delta_{E O}=|P(\hat{y}=1 \mid y=1, s=-1)-P(\hat{y}=1 \mid y=1, s=1)|$,
where $y$ represents the ground truth label, and $\hat{y}$ is the predicted label. Here, statistical parity is a measure for the independence of positive rate from the sensitive attribute, and equal opportunity signifies the level of the independence of true positive rate from the sensitive
attribute. Lower values for $\Delta_{S P}$ and $\Delta_{E O}$ indicate better fairness performance \cite{fairgnn} and are more desirable.\vspace{2pt}

\noindent \textbf{Implementation details. }  
\label{subsec:implementation}
We evaluate the proposed filter designs in two different environments. First, they are employed as bias mitigation sub-layers to filter the input representations to GNN layers in a two-layer graph convolutional network (GCN) \cite{gcn}; see also Figure \ref{fig:ex}. 
The GCN model is trained for node classification by employing the negative log-likelihood function as the objective. 
For this setting, the training set consists of $40\%$ of the nodes, while the remaining nodes are evenly split to create validation and test sets. The hyperparameter $\tau$ is selected via grid search among the values $\{0.0003, 0.0004, 0.0005, 0.0006\}$ for the filters $\tilde{\mathbf{h}}^{\text{f}}$ and $\tilde{\mathbf{h}}^{\text{f}}_{cf}$. Specifically, for $\tilde{\mathbf{h}}^{\text{f}}$, $\tau$ is chosen to be $0.0005$ and $0.0003$ on Pokec-z and Pokec-n, respectively, while it equals $0.0004$ for $\tilde{\mathbf{h}}^{\text{f}}_{cf}$ on both datasets. Moreover, for the proposed polynomial filter, $L$ is selected as $40$ and $50$ on datasets Pokec-z and Pokec-n, respectively, based on a grid search among the values $\{30, 40, 50\}$. To alleviate the hyperparameter tuning step for $\mathbf{h}^{\text{f}}$, $\tau=0.0004$ is directly utilized on both datasets without any fine-tuning. 

Second, to illustrate the use of the fair filters as post-processing operators, we use them to filter the predicted nodal labels computed by the classification algorithm presented in \cite{lp}. In the filtered signal, the components that are larger than a threshold are assigned to the first class, while the others are assigned to the second class. Note that this threshold is selected to be $0$ for labels $-1$ and $1$ in the experiments, however it can be adaptively chosen based on the input graph. In this second setting, $40\%$ of the nodes are used to train the model and the remaining ones contribute to the test set. The hyperparameter tuning process is kept the same as in the case where the filters are employed as pre-processing operators. For $\tilde{\mathbf{h}}^{\text{f}}$, $\tau$ is chosen to be $0.0004$ on both datasets, while it equals to $0.0004$ and $0.0006$ for $\tilde{\mathbf{h}}^{\text{f}}_{cf}$ on Pokec-z and Pokec-n, respectively. For the polynomial filter, $L$ is selected again as $40$ and $50$ on datasets Pokec-z and Pokec-n.

For all experiments, results are obtained for five random data splits, and their average along with the standard deviations are reported in the tables that follow. Further implementation details can be found in the publicly available code shared as supplementary material to this paper, which can be used to generate all results reported in this section.\vspace{2pt}

\noindent \textbf{Baselines. }  
 Fairness-aware baselines in the experiments include adversarial regularization \cite{fairgnn}, {\color{mpurple} EDITS \cite{dong2022edits}, and FairDrop \cite{spinelli2021fairdrop}. Adversarial regularization is a widely utilized fairness enhancement strategy, where an adversary is trained to predict the sensitive attributes. For adversarial regularization, the multiplier of the regularizer is tuned via a grid search among the values $\{ 0.1, 1, 10, 100, 1000\}$ (the multiplier of classification loss is assigned to be $1$). Furthermore, EDITS \cite{dong2022edits} is a model-agnostic debiasing framework that mitigates the bias in attributed networks before they are fed into any GNN. Specifically, it creates debiased versions of the nodal attributes and the graph structure, which are then input to the GCN network used here for node classification. For EDITS, the threshold proportion is tuned among the values $\{ 0.015, 0.02, 0.06, 0.29\}$, where these values are the optimized thresholds for other datasets used in \cite{dong2022edits}. Finally, FairDrop \cite{spinelli2021fairdrop} proposes a biased edge dropout strategy for a more balanced graph topology in terms of the edges connecting different (and the same) sensitive groups. The hyperparameter $\delta$ in the FairDrop algorithm is tuned among the values $\{ 0.7, 0.8, 0.9\}$}. 

\begin{table*}[h!]
	\centering
\caption{Adaptive Filter, $\tilde{\mathbf{h}}^{\text{f}}$ as Fairness-aware Post-processing Operator.}

\label{table:post}
\begin{scriptsize}
\begin{tabular}{l c c c c c c}
\cline{2-7}
\toprule
                                                    & \multicolumn{3}{{c}}{Pokec-z} & \multicolumn{3}{{c}}{Pokec-n}                                   \\ 
\cmidrule(r){2-7}
                       & Accuracy ($\%$) & $\Delta_{S P}$ ($\%$) & $\Delta_{E O}$ ($\%$)  & Accuracy ($\%$) & $\Delta_{S P}$ ($\%$) & $\Delta_{E O}$ ($\%$) \\\midrule
{\cite{lp}}
                   & $ \mathbf{64.83} \pm 0.54$ & $8.33 \pm 2.64$  & $9.38 \pm 2.54$ & $ 65.44 \pm 0.42$ & $6.27 \pm 4.83$  & $8.78 \pm 6.18$
    \\ \cmidrule(r){1-7}  
    {\cite{lp} +  $\tilde{\mathbf{h}}^{\text{f}}$ } 
                  & $  64.44 \pm 0.38$ & $\mathbf{1.58} \pm 1.01$  & $1.69 \pm 1.41$  & $  65.75 \pm 0.91$ & $\mathbf{2.11} \pm 2.19$  & $\mathbf{3.48} \pm 3.44$    \\
                  \cmidrule(r){1-7} 
                  {\cite{lp}$ +  \tilde{\mathbf{h}}_{cf}^{\text{f}}$} 
                   & $  64.70 \pm 0.48$ & $\mathbf{1.57} \pm 1.24$  & $\mathbf{1.55} \pm 1.21$   & $  \mathbf{65.80} \pm 0.86$ & $2.27 \pm 2.14$  & $\mathbf{3.48} \pm 3.34$  \\ \cmidrule(r){1-7} 
                  {\cite{lp} +   $\mathbf{h}^{\text{f}}$ } 
                  & $  64.62 \pm 0.54$ & $5.19 \pm 2.39$  & $6.09 \pm 3.00$  & $  \mathbf{65.78} \pm 0.93$ & $4.90 \pm 3.28$  & $6.44 \pm 5.48$\\

                  
\bottomrule
\end{tabular}
\end{scriptsize}
\end{table*}

\begin{table*}[ht]
	\centering
\caption{Ablation study for the employment of $\tilde{\mathbf{h}}^{\text{f}}$ as bias mitigation layers.}

\label{table:ablation}
\begin{scriptsize}
\begin{tabular}{l c c c c c c}
\toprule
                                                    & \multicolumn{3}{{c}}{Pokec-z}  & \multicolumn{3}{{c}}{Pokec-n}                                 \\ 
\cmidrule(r){2-4} \cmidrule(r){5-7}
                       & Accuracy ($\%$) & $\Delta_{S P}$ ($\%$) & $\Delta_{E O}$ ($\%$) & Accuracy ($\%$) & $\Delta_{S P}$ ($\%$) & $\Delta_{E O}$ ($\%$) \\\midrule
{GNN} 
                   & $ \textbf{66.52} \pm 0.27$ & $6.79 \pm 2.45$  & $7.26 \pm 3.29$ & $ 64.96 \pm 0.19$ & $6.79 \pm 2.45$  & $7.26 \pm 3.29$ 

\\ \cmidrule(r){1-7}  
 {  $\tilde{\mathbf{h}}^{\text{f}}+ $ GNN } 
                  & $  66.05 \pm 0.30$ & $\mathbf{1.08} \pm 1.20$  & $2.20 \pm 2.06$  & $  \mathbf{65.07} \pm 0.21$ & $2.12 \pm 1.01$  & $\mathbf{2.42} \pm 1.96$

                    \\ \cmidrule(r){1-7}  
 {  $\tilde{\mathbf{h}}^{\text{f}}$ before first layer } 
                  & $  66.22 \pm 0.23$ & $1.33 \pm 1.00$  & $\mathbf{1.98} \pm 2.18$  & $  \mathbf{65.05} \pm 0.31$ & $2.49 \pm 1.08$  & $2.55 \pm 2.32$
                  \\ \cmidrule(r){1-7} 
{  $\tilde{\mathbf{h}}^{\text{f}}$ before second layer } 
                  & $  66.17 \pm 0.24$ & $1.13 \pm 1.26$  & $2.06 \pm 1.80$  & $  \mathbf{65.10} \pm 0.18$ & $\mathbf{2.05} \pm 1.09$  & $2.46 \pm 1.91$
                  \\  
                    

\bottomrule
\end{tabular}
\end{scriptsize}
\end{table*}

\begin{table*}[ht]
	\centering
\caption{Sensitivity analysis for the hyperparameter $\tau$ in $\tilde{\mathbf{h}}^{\text{f}}$ as a bias mitigation layer.}

\label{table:sens11}
\begin{scriptsize}
\begin{tabular}{l c c c c c c}
\toprule
                                                    & \multicolumn{3}{{c}}{Pokec-z}  & \multicolumn{3}{{c}}{Pokec-n}                                 \\ 
\cmidrule(r){2-4} \cmidrule(r){5-7}
                       & Accuracy ($\%$) & $\Delta_{S P}$ ($\%$) & $\Delta_{E O}$ ($\%$) & Accuracy ($\%$) & $\Delta_{S P}$ ($\%$) & $\Delta_{E O}$ ($\%$) \\\midrule
{GNN} 
                   & $ 66.52 \pm 0.27$ & $6.79 \pm 2.45$  & $7.26 \pm 3.29$ & $ 64.96 \pm 0.19$ & $6.79 \pm 2.45$  & $7.26 \pm 3.29$ 
\\ \cmidrule(r){1-7}  

{$\tau=0.0003$} 
                   & $ { 66.33 \pm 0.25}$ & $1.34 \pm 1.39$  & ${2.26 \pm 2.07}$ & $ {\color{mpurple} \mathbf{65.07} \pm 0.21}$ & $\mathbf{2.12} \pm 1.01$  & $\mathbf{2.42} \pm 1.96$   
    \\ \cmidrule(r){1-7} 
{$\tau=0.0004$} 
                   & $  66.33 \pm 0.22$ & $1.15 \pm 1.33$  & $2.26 \pm 1.75$ & $  65.04 \pm 0.20$ & $2.18 \pm 0.95$  & $2.47 \pm 1.87$   
    \\ \cmidrule(r){1-7}  
 
{$\tau=0.0005$} 
                  & $  66.05 \pm 0.30$ & $\mathbf{1.08} \pm 1.20$  & $\mathbf{2.20} \pm 2.06$  & $  \mathbf{65.05} \pm 0.18$ & $2.41 \pm 0.86$  & $2.82 \pm 1.68$ \\ \cmidrule(r){1-7}  
{$\tau=0.0006$} 
                  & $ \mathbf{ 66.76} \pm 0.25$ & $1.49 \pm 1.27$  & $2.73 \pm 2.27$  & $  64.97 \pm 0.12$ & $2.46 \pm 0.66$  & $2.93 \pm 1.58$ \\ 
\bottomrule
\end{tabular}
\end{scriptsize}
\end{table*}
\begin{table*}[ht]
	\centering
\caption{Sensitivity analysis for the hyperparameter $\tau$ in $\tilde{\mathbf{h}}_{cf}^{\text{f}}$ as a bias mitigation layer.}

\label{table:sens12}
\begin{scriptsize}
\begin{tabular}{l c c c c c c}
\toprule
                                                    & \multicolumn{3}{{c}}{Pokec-z}  & \multicolumn{3}{{c}}{Pokec-n}                                 \\ 
\cmidrule(r){2-4} \cmidrule(r){5-7}
                       & Accuracy ($\%$) & $\Delta_{S P}$ ($\%$) & $\Delta_{E O}$ ($\%$) & Accuracy ($\%$) & $\Delta_{S P}$ ($\%$) & $\Delta_{E O}$ ($\%$) \\\midrule
{GNN} 
                   & $ \textbf{66.52} \pm 0.27$ & $6.79 \pm 2.45$  & $7.26 \pm 3.29$ & $ 64.96 \pm 0.19$ & $6.79 \pm 2.45$  & $7.26 \pm 3.29$ 
\\ \cmidrule(r){1-7}  
{$\tau=0.0003$} 
                   & $  66.30 \pm 0.24$ & $1.34 \pm 1.38$  & $2.34 \pm 2.07$ & $  \mathbf{65.07} \pm 0.21$ & $\mathbf{2.12} \pm 1.01$  & $\mathbf{2.42} \pm 1.96$   
    \\ \cmidrule(r){1-7}  
{$\tau=0.0004$} 
                   & $  66.34 \pm 0.27$ & $1.23 \pm 1.43$  & $\mathbf{2.15} \pm 1.96$ & $  \mathbf{65.05} \pm 0.21$ & $\mathbf{2.13} \pm 0.93$  & $\mathbf{2.39} \pm 1.78$   
    \\ \cmidrule(r){1-7}  
{$\tau=0.0005$} 
                   & $ { 66.34 \pm 0.22}$ & ${\mathbf{1.19} \pm 1.36}$  & ${2.35 \pm 1.74}$ & $ {\color{mpurple} 64.99 \pm 0.19}$ & $2.16 \pm 0.82$  & $\mathbf{2.39} \pm 1.84$   
    \\ \cmidrule(r){1-7}  
{$\tau=0.0006$} 
                   & $  66.19 \pm 0.36$ & $1.66 \pm 1.10$  & $2.64 \pm 2.27$ & $  65.01 \pm 0.10$ & $2.28 \pm 0.92$  & $2.58 \pm 1.83$  \\ 
\bottomrule
\end{tabular}
\end{scriptsize}
\end{table*}


\subsection{Results}
Comparative results for the proposed fairness-aware graph filters, $\tilde{\mathbf{h}}^{\text{f}}$, $\tilde{\mathbf{h}}_{cf}^{\text{f}}$, and $\mathbf{h}^{\text{f}}$ are presented in Table \ref{table:gnn}, for the case where they are utilized as bias mitigation layers. The natural baseline for the proposed strategies is to employ the GNN model without any fairness-aware operations, where this scheme is denoted by ``GNN'' in Table \ref{table:gnn}. Moreover, ``Adversarial'', ``EDITS'', and ``FairDrop'' in Table \ref{table:gnn} stand for the adoption of adversarial regularization in training \cite{fairgnn}, and state-of-the-art fairness-aware baselines EDITS \cite{dong2022edits}, and FairDrop \cite{spinelli2021fairdrop}, respectively. 

The results in Table \ref{table:gnn} demonstrate that all of the proposed filter designs improve upon the  naive GNN baseline in terms of both fairness metrics, while also providing similar utility. Specifically, the proposed strategies achieve $30\%$ to $90\%$ improvement in all fairness measures for every dataset compared to GNN. The results further demonstrate the superior fairness performance of $\tilde{\mathbf{h}}^{\text{f}}$ over the polynomial filter $\mathbf{h}^{\text{f}}$, which is expected, as $\tilde{\mathbf{h}}^{\text{f}}$ can better optimize our bias metric $\rho$ with a higher number of degrees of freedom compared to $\mathbf{h}^{\text{f}}$. Moreover, it is observed that our designs, $\tilde{\mathbf{h}}^{\text{f}}$ and $\tilde{\mathbf{h}}_{cf}^{\text{f}}$, lead to similar fairness improvements, which signifies that the derived upper bound in \eqref{eq:mfin} is a successful surrogate bias measure for the Pokec graphs.

The results in Table \ref{table:gnn} also show that $\tilde{\mathbf{h}}^{\text{f}}$ and $\tilde{\mathbf{h}}_{cf}^{\text{f}}$ always achieve better fairness performance together with similar/better utility, compared to other fairness-aware baselines, namely Adversarial \cite{fairgnn}, EDITS \cite{dong2022edits} and FairDrop \cite{spinelli2021fairdrop}. While the polynomial filter, $\mathbf{h}^{\text{f}}$ generally leads to a similar fairness improvement compared to other fairness-aware baselines, this fairness performance is typically accompanied by a better utility for $\mathbf{h}^{\text{f}}$. Furthermore, it can be observed that the employment of the novel filters generally leads to the lowest standard deviation values, and therefore enhances the stability of the results. Overall, the results corroborate the efficacy of the proposed filter designs design in mitigating bias, while also providing similar utility measures compared to the state-of-the-art fairness-aware baselines.

Fairness performance in Table \ref{table:gnn} is reported in terms of commonly utilized group fairness measures; namely, statistical parity and equal opportunity, same as prior works \cite{dong2022edits, spinelli2021fairdrop, fairgnn}. In Table \ref{table:corr}, we also provide the total correlation values between the sensitive attributes and representations that are input to or output from the designed filter $\tilde{\mathbf{h}}^{\text{f}}$. 
With reference to the two-layer GNN architecture in Figure \ref{fig:ex} that is used for this experiment, in the first row of Table \ref{table:corr} we report $\|\mathbf{s}^\top \mathbf{X}\|_1$ (before $\tilde{\mathbf{h}}^{\text{f}}$) and $\|\mathbf{s}^\top \bar{\mathbf{X}}\|_1$ (after $\tilde{\mathbf{h}}^{\text{f}}$). Likewise, in the second row, we report the total correlations $\|\mathbf{s}^\top \mathbf{H}_1\|_1$ (before $\tilde{\mathbf{h}}^{\text{f}}$) and $\|\mathbf{s}^\top \bar{\mathbf{H}}_1\|_1$ (after $\tilde{\mathbf{h}}^{\text{f}}$). Overall, the results demonstrate that $\tilde{\mathbf{h}}^{\text{f}}$ can significantly reduce the correlation that is expected to lead to intrinsic bias, which is also reflected in the improved $\Delta_{SP}$ and $\Delta_{EO}$ values in Table \ref{table:gnn}. This correlation reduction is observed at both stages in this two-layer GCN and for both datasets. Notice that $\|\mathbf{s}^\top \mathbf{H}_1\|_1>\|\mathbf{s}^\top \bar{\mathbf{X}}\|_1$ because the GCN layer (mapping $\bar{\mathbf{X}}$ to $\mathbf{H}_1$) aggregates information using $\hat{\mathbf{A}}$, and the latter is highly correlated with $\mathbf{s}$ as discussed in Section \ref{subsec:met}. Furthermore, by comparing the first and second rows in Table \ref{table:corr} it is observed that the correlation reduction is more pronounced before any GNN layer is used to process the data. Since the representations output by a GNN layer are learned to maximize utility, this phenomenon is an expected result of the corresponding fairness-utility tradeoff. 


We also provide experimental results herein, whereby the proposed fair filters are employed as post-processing operators on the predicted labels (a graph signal) of a node classification algorithm. For this setting, the classification results are obtained via the algorithm presented in \cite{lp}, and we subsequently filter these predicted labels to debiase them. The results are presented in Table \ref{table:post}, which exhibit similar tendencies as those in Table \ref{table:gnn}. Overall, our experiments confirm the efficacy of the proposed graph filters in improving fairness measures and also for the setting where they are employed as post-processing operators. In addition, similar to the findings of Table \ref{table:gnn}, better fairness measures are typically accompanied by better stability and similar utility to the fairness-agnostic baseline \cite{lp}. \vspace{2pt}

\noindent \textbf{Ablation study.} To examine the effect of filter placement in the adopted two-layer GCN, we carry out an ablation study whose results are presented in Table \ref{table:ablation}. Therein, ``$\tilde{\mathbf{h}}^{\text{f}} + $GNN'' corresponds to an architecture where the designed filter is employed before both of the GCN layers; exactly as in Figure \ref{fig:ex}. In the meantime, ``$\tilde{\mathbf{h}}^{\text{f}}$ before first layer'' and  ``$\tilde{\mathbf{h}}^{\text{f}}$ before second layer'' denote architectures that utilize \emph{a single filter} placed before the first layer only, or, the second layer only, respectively. Naturally, ``GNN'' corresponds to a baseline model which does not employ bias-mitigating filters. The key conclusion from this study is that using at least one filter, regardless of its placement within the architecture, always helps improve fairness measures. Furthermore, if only one filter is used, we find that placing it deeper (meaning before the second layer) results in better/similar fairness measures compared to an earlier placement of the filter. Finally, results in Table \ref{table:ablation} are inconclusive as to whether employing the filter before all layers is always the best strategy due to the high variances. Still, we find that employing the proposed filter before every layer achieves a similar fairness performance in the worst case compared to single filter placements. Thus, we suggest the use of filters in all layers for a simpler design. 

\noindent \textbf{Sensitivity analyses.} For our designs $\tilde{\mathbf{h}}^{\text{f}}$ and $\tilde{\mathbf{h}}_{cf}^{\text{f}}$,  sensitivity analyses are presented in Tables \ref{table:sens11} and \ref{table:sens12}, respectively; in order to assess their sensitivity to their hyperparameter, $\tau$, for the case where they are employed as bias mitigation layers. Note that Figure \ref{fig:within} suggests that the number of frequencies where the magnitudes of $\tilde{\mathbf{s}}$ are markedly higher than $\tilde{\mathbf{y}}$ is around $3$ for both datasets. Thus, the range of $\tau$ is chosen so that the total number of spectral components for which the filters' frequency response is approximately equal to $0$ is less than $10$. This way, we expect to improve markedly in terms of fairness without incurring a major degradation in utility. Overall, the results demonstrate that the filters, $\tilde{\mathbf{h}}^{\text{f}}$ and $\tilde{\mathbf{h}}_{cf}^{\text{f}}$, always lead to better fairness measures compared to the fairness-agnostic GNN baseline, within a broad range of hyperparameter choices. The sensitivity analyses for setting where the filters are used as post-processing operators are deferred to the Appendix, which lead to a similar conclusion. 

\subsection{On the effective network operator}

\begin{figure}[t]
    \centering
        \includegraphics[width=\linewidth]{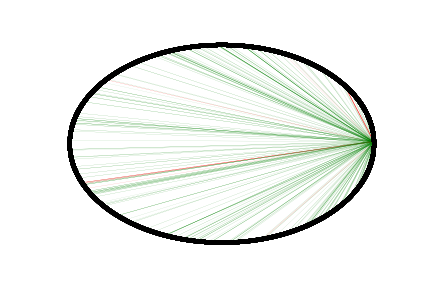}
        \includegraphics[width=\linewidth]{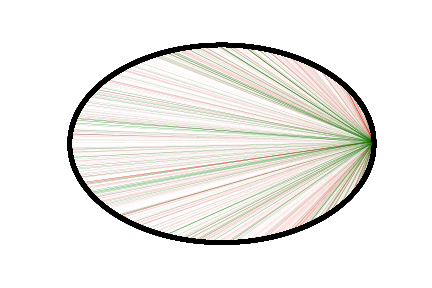}
  \caption{ For a sampled subgraph from the Pokec network, the distribution of the intra-edges (green) and inter-edges (red) in the effective network topology without (top)/ with (bottom) the application of $\tilde{\mathbf{h}}^{\text{f}}$.}   
    \label{fig:effective}
\end{figure}

Based on \eqref{eq:effec}, the effective network operator used in the learning process is defined to be $\mathbf{\bar{A}}:=\mathbf{V} (\mathbf{I}_{N} - \boldsymbol{\Lambda} )\text{diag}(\tilde{\mathbf{h}}) \mathbf{V}^{\top}$. Therefore, employing the proposed graph filters can be interpreted as a modification to the original graph connectivity with the final aim of reducing the structural bias. For graphs encountered in various application domains, it is typically observed that the number of edges connecting the same sensitive groups, intra-edges, is significantly larger than the number of edges linking different sensitive groups, inter-edges, due to the homophily principle \cite{kose2022fair, fairgnn}. Based on this observation, several studies have demonstrated that the imbalance between the number of intra and inter-edges is a major factor for the resulting algorithmic bias \cite{dyadic, spinelli2021fairdrop, kose2022fair}. Motivated by this, here we visualize the intra- and inter-edges in a sub-graph extracted from the Pokec network and their distributional change in the effective network operator after applying the filter $\tilde{\mathbf{h}}^{\text{f}}$. Specifically, Figure \ref{fig:effective} illustrates the intra- and inter-edges using the colors green and red, respectively, for the original subgraph and the modified effective graph structure after $\tilde{\mathbf{h}}^{\text{f}}$ is employed. The figure reveals that the application of $\tilde{\mathbf{h}}^{\text{f}}$ has a balancing effect in the number of intra- and inter-edges in the resulting graph structure, which can help visualize the bias mitigation mechanisms of the proposed strategies. Note that this balancing effect is also supported by comparing the total weights of intra- and inter-edges in the original network operator $\mathbf{\hat{A}}$ versus the effective one $\mathbf{\bar{A}}$ obtained when accounting for the filters. Specifically, in the original topology, the total weights of the intra- and inter-edges are $3102$ and $165$, respectively. On the other hand, the application of $\tilde{\mathbf{h}}^{\text{f}}$ has a balancing effect resulting in $1824$ and $1424$ intra- and inter-edges, respectively.

\section{Conclusion}
In this study, we put forth three novel graph filter designs with the goal of mitigating bias stemming from the graph topology. Specifically, we first introduce a bias metric, $\rho$, that is applicable to unsupervised learning settings and measures the correlation between the connectivity pattern and sensitive attributes. Our first graph filter design, $\tilde{\mathbf{h}}^{\text{f}}$, is obtained as a solution to a convex optimization problem that minimizes $\rho$. For a more efficient solution, we carry out a bias analysis and formulate an LP problem that targets the minimization of an upper bound on $\rho$. Remarkably, we show the LP attains a closed-form optimal solution for a fair graph filter $\tilde{\mathbf{h}}^{\text{f}}_{cf}$. Finally, we take a fair, polynomial graph convolution filter, $\mathbf{h}^{\text{f}}$, into consideration, where the number of optimization variables in the corresponding design is independent of the input graph size. The proposed fairness-aware graph filters can be flexibly employed in various graph-based ML and SP algorithms at different stages of learning. Node classification experiments on real-world networks demonstrate that all of the proposed filter designs mitigate bias effectively. We observe they typically lead to better fairness measures when compared to other state-of-the-art fairness-aware baselines, and without notably sacrificing utility (i.e., classification accuracy).

This work opens up several exciting future directions. First, the proposed designs assume the existence of a single sensitive attribute, whereas considering multiple sensitive attributes in our designs would be certainly of interest. Second, this study focuses on the linear correlation between the graph structure and sensitive attributes as a bias measure, and extending our analysis to non-linear correlation metrics is another important future direction. Finally, robust adaptations of the proposed designs to accommodate several real-world challenges, including but not limited to missing sensitive attribute/graph structure information, and privacy constraints, are important components of our future research agenda.

%

\appendices
\section{Further Sensitivity Analyses}
\label{app:sens}
\begin{table*}[h!]
	\centering
\caption{Sensitivity analysis for the hyperparameter $\tau$ in $\tilde{\mathbf{h}}^{\text{f}}$ as a post-processor.}

\label{table:sens21}
\begin{scriptsize}
\begin{tabular}{l c c c c c c}
\cline{2-7}
\toprule
                                                    & \multicolumn{3}{{c}}{Pokec-z} & \multicolumn{3}{{c}}{Pokec-n}                                   \\ 
\cmidrule(r){2-7}
                       & Accuracy ($\%$) & $\Delta_{S P}$ ($\%$) & $\Delta_{E O}$ ($\%$)  & Accuracy ($\%$) & $\Delta_{S P}$ ($\%$) & $\Delta_{E O}$ ($\%$) \\\midrule
{\cite{lp}}
                   & $ \mathbf{64.83} \pm 0.54$ & $8.33 \pm 2.64$  & $9.38 \pm 2.54$ & $ 65.44 \pm 0.42$ & $6.27 \pm 4.83$  & $8.78 \pm 6.18$
    \\ \cmidrule(r){1-7}  
{$\tau=0.0003$} 
                   & $  64.40 \pm 0.33$ & $1.67 \pm 0.95$  & $1.83 \pm 1.30$   & $  65.48 \pm 0.50$ & $2.33 \pm 2.27$  & $3.85 \pm 3.58$    \\\cmidrule(r){1-7}  
{$\tau=0.0004$} 
                   & $  64.44 \pm 0.38$ & $\mathbf{1.58} \pm 1.01$  & $\mathbf{1.69} \pm 1.41$   & $  \mathbf{65.75} \pm 0.91$ & $\mathbf{2.11} \pm 2.19$  & $\mathbf{3.48} \pm 3.44$    \\\cmidrule(r){1-7} 
{$\tau=0.0005$} 
                   & $  64.45 \pm 0.42$ & $1.94 \pm 0.81$  & $2.28 \pm 1.57$   & $  \mathbf{65.78} \pm 0.91$ & $\mathbf{2.14} \pm 2.20$  & $\mathbf{3.50} \pm 3.47$    \\\cmidrule(r){1-7}  
{$\tau=0.0006$} 
                   & $  64.27 \pm 0.48$ & $2.01 \pm 1.17$  & $2.63 \pm 1.85$   & $  65.81 \pm 0.88$ & $2.20 \pm 2.18$  & $\mathbf{3.51} \pm 3.44$    \\   
\bottomrule
\end{tabular}
\end{scriptsize}
\end{table*}

\begin{table*}[h!]
	\centering
\caption{Sensitivity analysis for the hyperparameter $\tau$ in $\tilde{\mathbf{h}}_{cf}^{\text{f}}$ as a post-processor.}

\label{table:sens22}
\begin{scriptsize}
\begin{tabular}{l c c c c c c}
\cline{2-7}
\toprule
                                                    & \multicolumn{3}{{c}}{Pokec-z} & \multicolumn{3}{{c}}{Pokec-n}                                   \\ 
\cmidrule(r){2-7}
                       & Accuracy ($\%$) & $\Delta_{S P}$ ($\%$) & $\Delta_{E O}$ ($\%$)  & Accuracy ($\%$) & $\Delta_{S P}$ ($\%$) & $\Delta_{E O}$ ($\%$) \\\midrule
{\cite{lp}}
                   & $ \mathbf{64.83} \pm 0.54$ & $8.33 \pm 2.64$  & $9.38 \pm 2.54$ & $ 65.44 \pm 0.42$ & $6.27 \pm 4.83$  & $8.78 \pm 6.18$
    \\ \cmidrule(r){1-7} 
    {$\tau=0.0003$} 
                   & $  \mathbf{64.80} \pm 0.39$ & $1.64 \pm 0.82$  & $1.76 \pm 1.00$   & $  65.68 \pm 0.33$ & $2.61 \pm 2.19$  & $5.57 \pm 2.91$    \\\cmidrule(r){1-7}  
{$\tau=0.0004$} 
                   & $  64.70 \pm 0.48$ & $\mathbf{1.57} \pm 1.24$  & $\mathbf{1.55} \pm 1.21$   & $  65.48 \pm 0.50$ & $2.33 \pm 2.04$  & $4.54 \pm 3.00$    \\\cmidrule(r){1-7}  
{$\tau=0.0005$} 
                   & $  64.35 \pm 0.50$ & $1.83 \pm 1.14$  & $2.37 \pm 0.47$   & $  \mathbf{65.79} \pm 0.90$ & $\mathbf{2.24} \pm 2.14$  & $3.55 \pm 3.40$    \\\cmidrule(r){1-7}  
{$\tau=0.0006$} 
                   & $  64.07 \pm 0.39$ & $2.31 \pm 1.70$  & $3.09 \pm 1.72$   & $  \mathbf{65.80} \pm 0.86$ & $\mathbf{2.27} \pm 2.14$  & $\mathbf{3.48} \pm 3.34$    \\   
\bottomrule
\end{tabular}
\end{scriptsize}
\end{table*}

The sensitivity analyses are further provided in Tables \ref{table:sens21} and \ref{table:sens22} for the case where the proposed filters, $\tilde{\mathbf{h}}^{\text{f}}$ and $\tilde{\mathbf{h}}_{cf}^{\text{f}}$, are employed in the post-processing step. The results in these tables signify that the proposed strategies always improve the natural baseline in terms of fairness for a wide range of $\tau$ values also for their employment as post-processing operators.

\bibliographystyle{IEEEtranS}
\bibliography{main}

\end{document}